\documentclass[submission,copyright,creativecommons]{eptcs}
\providecommand{\event}{QPL 2022} 
\usepackage{breakurl}             
\usepackage{underscore} 

\usepackage{amsmath, amssymb, amsthm}
\usepackage{xcolor, graphicx, subcaption, multirow}

\newtheorem{prop}{Proposition}

\title{The Causal Structure of Semantic Ambiguities}
\author{Daphne Wang
\institute{University College London}
\and
Mehrnoosh Sadrzadeh
\institute{University College London}
}
\def\titlerunning{The Causal Structure of Semantic Ambiguities}
\def\authorrunning{D. Wang \& M. Sadrzadeh
}
\begin{document}
\maketitle

\begin{abstract}
    Ambiguity is a natural language phenomenon occurring at different levels of syntax, semantics, and pragmatics. It is widely studied; in Psycholinguistics, for instance,  we have a variety of competing studies for the human disambiguation processes. These studies are empirical and based on eye-tracking measurements.  Here we take first steps towards formalizing these processes for semantic ambiguities where we identified the presence of two features: (1) joint plausibility degrees of different possible interpretations, (2) causal structures according to which certain words play a more substantial role in the processes. The novel sheaf-theoretic model of definite causality developed by Gogioso and Pinzani in QPL 2021 offers tools to model and reason about these features. We applied this theory to a dataset of ambiguous phrases extracted from Psycholinguistics literature and their human plausibility judgements collected by us using the Amazon Mechanical Turk engine. We measured the causal fractions  of different disambiguation orders within the phrases and discovered  two prominent orders: from subject to verb in the subject-verb and from object to verb in the verb object phrases.  We also found evidence for delay in the disambiguation of polysemous vs homonymous verbs, again compatible with Psycholinguistic findings. 
\end{abstract}

\section{Introduction}
\label{sec:intro}

Discovering, studying, and formalising the ambiguities of natural language  is an area of the field of Computational Linguistics.  Natural language ambiguities occur  at different levels of syntax, semantics, and pragmatics. Syntactic ambiguities are due to single words with multiple grammatical roles, e.g. both noun and verb, as in `book' and `cook'. They also occur at the phrase level, e.g. in the phrase `old men and women'  the adjective `old' can be modifying both the word `men' or the conjunction `men and women'. Other  examples are  `Show me the meals on the flight from SF', and `We saw the Eiffel Tower flying to Paris'.   Ambiguities  surpass sentential boundaries, as in the anaphorically ambiguous discourse `I put the CD in the computer. It broke.'.  Another  prominent type of ambiguity is due to different interpretations of words and phrases. Words/phrases that are `homonymous'  have more than one unrelated  interpretation, due to historical incidence or other reasons, as in `plant', `pitcher', or `coach'. Polysemous words/phrases have more than one related interpretations,  as in `newspaper', which can refer to the collection of papers stapled together (literal) or the content conveyed by these (figurative). 

The process of semantic disambiguation and the role of the context has been studied in the field of Psycholinguistics via devices such as eye-tracking. Here, the delay in reading and the trajectory of the gaze shows that  context plays different roles when disambiguating different  types of semantically ambiguous words. These results show that ambiguous verbs get disambiguated late, e.g. after reading the whole sentence, whereas certain ambiguous nouns get disambiguated almost immediately and without much reliance on context. They also show that polysemous verbs are disambiguated  later than homonymous verbs, and the same applies for nouns, i.e. that polysemous nouns are disambiguated later than homonymous nouns.   Semantic ambiguities are plentiful and universal.   They also easily lift from word to phrase level, where the number and complexity of potential interpretations  increase exponentially. Assuming only two interpretations per word (which is the minimal criteria for ambiguity; resources such as \texttt{WordNet} list 78695 senses for  a total of 128321 words of English), a 2-word phrase, the simplest  ambiguous combination of two ambiguous words, can have up to four different interpretations, a 3-word phrase up to 8, an $n$-word one up to $2^n$.  Instances of ambiguous 2-word phrases are `(the) cabinet reflects', `(the) pitcher threw',  and `(the) plant bored', where  `pitcher threw'  has 3 interpretations: a jug throwing a shadow, a baseball player throwing a shadow, or a baseball player throwing a ball. `A plant bore' has 4 interpretations: a factory that makes holes (e.g. in metal bars), a factory that made its workers weary, a house plant that was uninteresting,  a house plant that pierced its pot.

 Whether word or phrase, the interpretations of ambiguous natural language expressions depends on their contexts and this can be formalised in different ways. We use the fact that  interpretations are context-dependent  and that this gives rise to context-dependent probability distributions,  corresponding to the likelihood that a certain meaning of a word is selected in a context. By  context we refer to any  linguistic or non-linguistic information, e.g. knowledge-based, resources and background information. This gives rise to two questions that we aim at investigating.  (1) Given these resources, can the single distributions of the interpretations of each word within a phrase be used to compute a distribution for the meaning combination of the full phrase? For example, if we see the word ``pitcher'' in the corpus mostly as a baseball player and ``threw'' as throwing a ball, when we next come across the phrase ``pitcher threw'', can we be sure that it means a baseball player threw a ball and not a jug threw a shadow? (2) Since there is a temporal order in the disambiguation process, is there a causal order in the process, and if so how  can it be quantified used to replicate the Psycholinguistics findings?  The first question can be formalised in terms of quantum-like contextuality and indeed previous research has been done on whether cognitive processes are contextual in this way, (see for example the work of Bruza et. al. focusing on concept combinations\cite{BRUZA201526} and the ``mental lexicon''\cite{Bruza2009}). 
 

We answered the first question in previous work   \cite{Wangetal-Conf,Wangetal-Abs,Wangetal-Jour}  using the sheaf theoretic model of contextuality of  \cite{AbramskyBrad} and its generalisation to signalling scenarios in the Contextuality by Default (CbD) setting of  \cite{Dzhafarovetal}. Using these tools, we  formalised  the first question as: ``Is there a global joint probability distribution that describes the probabilistic distributions of phrases where we can maximise the probability distributions of each word within the phrase?".  We hypothesised that, similar to the case in Quantum mechanics,  the answer to this question is no, and found a few examples that witnessed it. This led us to the conclusion that the pre-existing value of the interpretation of a word in a phrase is not independent of the interpretations of other words in the context (including the phrase itself).

In this paper, we formalise and answer the second question using recent advances in  causal sheaf theory and in particular the development  of \cite{sheafcausality}.  Our methodology is as follows. We first devise a dataset  consisting of equal numbers of polysemous and homonymous nouns and verbs. The nouns and verbs are trimmed down from a larger such set, with the demarcation rule that both of their subject-verb and verb-object combinations in a phrase would make sense. We put these phrases on Amazon Mechanical Turk and collect human judgements for degrees of plausibility of each phrase.  We compute probability distributions from our Amazon Turk human judgements and verify which proportion of the judgements are compatible with one of our four main causal orders: Object/Subject $\to$ Verb, Verb $\to$ Object/Subject. We then work within Subject-Verb and Verb-Object phrases and for each phrase type study which of the four causal order was higher than the other:   polysemous verbs/nouns or homonymous verbs/nouns.  Our findings confirm the Psycholinguistic research, that (1) the prominent causal order of phrases is from the noun to the verb, i.e. from Subject $\to$ Verb in Subject-Verb phrases and from the Object $\to$ Verb in the Verb-Object phrase. In other words, the verb is the last part of speech to be disambiguated in a sentence,  (2) polysemous verbs are disambiguated  later than homonymous verbs, (3) polysemous nouns are disambiguated later than homonymous nouns.

\section{The causal framework and the causal fraction}
The aim of this section is to introduce the causal framework  and the causal fraction therein. In this entry paragraph we briefly review the relevance of these to linguistic scenarios. 

 The study of quantum contextuality  relies on the \emph{no-signalling} property of a given system; i.e. that choices of observables do not influence the outcomes of other observables measured at the same time. However, linguistic scenarios do not necessarily  satisfy this  property, and neither does it fit with our intuition: no-signalling in our linguistic scenarios would mean that the choice of a word does not influence the interpretation of other words in a phrase. There are mathematical models  that generalise the notion of contextuality from no-signalling to all systems. An example is the Contextual-by-Default  (CbD) framework \cite{Dzhafarovetal}. In this framework, however, the source and nature of  signalling propoerty is completely disregarded, and the interpretation of what contextuality might say about the system is not  clear anymore. In previous work \cite{Wangetal-Abs,Wangetal-Conf,Wangetal-Jour},  we nonetheless used the CbD framework and analysed ambiguous examples from natural language. Here, instead,  we  work with the extension of the sheaf-theoretic model of contextuality to causality. This extension was developed in \cite{sheafcausality}. The use of the causal framework not only enables us to  allow for some signalling, but also determines whether the signalling observed in our systems has a direction; for example, does the choice of verb has an influence on the interpretation of its object, or similarly on the interpretation of its subject.  

\subsection{The causal framework}
\paragraph{}
A causal scenario consists  of a list of events and a set of causal relations associated to them. An event $X$ is usually considered to be a generalised process with a set of possible inputs and outputs; these are respectively denoted as $I_X$ and $O_X$. If we are considering a definite causal scenario, then the relations form a partial order where, e.g. $A \to B$ if $A$ causally precedes $B$, as we do not allow causal loops (antisymmetry), and causality is clearly transitive and reflexive.  If we are considering  indefinite causal scenarios, then more exotic processes not compatible with the standard circuit model of quantum computation will be needed to describe the systems.

Formally, a \emph{causal scenario} is defined to be a triple $\Sigma = \left(\Omega, \underline{I}, \underline{O}\right)$ where $\Omega$ is a poset representing the causal relations between events, $\underline{I} = \left(I_\omega\right)_{\omega\in \Omega}$ includes all possible inputs for all events, and  $\underline{O} = \left(O_\omega\right)_{\omega\in \Omega}$ includes all the possible outputs for all events.

\paragraph{}Given the poset $\Omega$, its associated set of \emph{lowersets}  is denoted by $\Lambda\left(\Omega\right)$. These are downwards closed subsets of $\Omega$. In terms of causal events, each element of $\Lambda\left(\Omega\right)$ corresponds to a set of events that admits a description independent to other events.

\paragraph{}The \emph{locale of events} is  the set:
\begin{equation}
    \mathcal{L}_\Sigma = \left\{\left(\lambda, \left(U_\omega\right)_{\omega\in\lambda}\right)~\middle|~\lambda\in \Lambda\left(\Omega\right), U_\omega\subseteq I_\omega, U_\omega\neq \emptyset\right\}
\end{equation}
with which there is associated a partial order $U = \left(\lambda_U, \underline{U}\right) \leq V  = \left(\lambda_V, \underline{V}\right)$ iff $\lambda_U\subseteq \lambda_V$ and $U_\omega \subseteq V_\omega$ for all $\omega\in \lambda_U$. The meets and joins of this locale are defined as follows:
\begin{align}
    U\cup V =& \left(\lambda_U\cup \lambda_V, \left(U_\omega \cup V_\omega\right)\right)\\
    U\cap V =& \left(\left\{\omega \in \lambda_U\cap\lambda_V ~\middle|~ U_\omega\cap V_\omega \neq \emptyset\right\}, \left(U_\omega\cap V_\omega\right)\right)
\end{align}

\paragraph{}Over a (causal) poset $\Omega$, we introduce \emph{causal functions}; these are the  functions from inputs to outputs that respect the causal order, i.e. the inputs of succeeding events do not influence the outputs of preceding events. Formally,  this means that if we order the events $\omega_1 \to \omega_2 \ldots \to \omega_n$,  the function $f: I_1\times I_2 \times \ldots \times I_n \to O_1\times O_2 \times \ldots O_n$ satisfies the following for all $k$'s:
\begin{equation}
    \left.f\left(a_1, \ldots, a_k, a_{k+1}, \ldots, a_n\right)\right|_{\omega_k\downarrow} = \left.f\left(a_1, \ldots, a_k, a'_{k+1}, \ldots, a'_n\right)\right|_{\omega_k\downarrow}
\end{equation}

\paragraph{}We can now define the \emph{event sheaf} as follows:
\begin{align}
    \mathcal{E}_\Sigma: \mathcal{L}^{op} \quad\to&\quad \mathbf{Set}\nonumber\\
    \left(\lambda_U, \underline{U}\right) \quad\mapsto&\quad \left\{f:\underline{U} \to \prod_{\omega\in \lambda_U} O_\omega ~\middle|~ f \text{ causal}\right\}
\end{align}
and the restriction map is given by:
\begin{equation}
    \left(U\leq V\right) \quad \mapsto \quad \left(f\mapsto f|_U\right)
\end{equation}
The above means that the event sheaf associates to a lowerset and a set of inputs on the lowerset,  the set of all functions which respect the causal order of the scenario.  This sheaf encodes the required conditions for definite causality.

\paragraph{}An \emph{empirical model} is  a specific distribution on causal functions and is defined over all strings of inputs. Formally, it is the following element:
\begin{equation}
    e \in \prod_{\underline{i}\in \mathcal{I}_\Sigma} \mathcal{D}_R \mathcal{E}_\Sigma \left(\Omega, \underline{i}\right)
\end{equation}
where $\underline{i}$ is -- by abus de langage --  the string of singletons $\left(\left\{i_\omega\right\}\right)_{i_\omega\in \underline{i}}$, and $\mathcal{I}_\Sigma = \prod_{I_\omega\in \Omega} I_\omega$ is the set of all strings of inputs (for all events). Also, $\mathcal{D}_R$ is the $R$-distribution monad and $\mathcal{E}_\Sigma$ is the event sheaf  defined  above. This needs to be a compatible family of distributions (by definition), i.e.  the marginals on lowersets should be well-defined. For the rest of this paper, we will only make use of probabilistic distributions (i.e. $R=\mathbb{R}_+$).

\subsubsection{Example\label{subsubsec:BelltypeEx}} Let's consider an empirical model with only two events $\Omega = \{A, B\}$ with the single causal relation $A\to B$, such that $I_A = I_B = O_A = O_B=\{0, 1\}$, and the probability distribution on $\Omega\in \Lambda(\Omega)$  given as follows:
\begin{equation}\label{eq:finalEMcausal}
    \scalebox{.8}{
    \begin{tabular}{cc|c|c|c|c|}
         \multicolumn{2}{c|}{\multirow{2}{*}{$(A,B)$}} &\multicolumn{4}{c|}{Output}\\\cline{3-6}
         & & $(0,0)$ & $(0,1)$ &$(1,0)$ & $(1,1)$ \\\hline
         \multicolumn{1}{c|}{\multirow{4}{*}{\rotatebox{90}{Input}}} & $(0,0)$ & $0$ & $6/13$ & $0$ & $7/13$\\\cline{2-6}
         \multicolumn{1}{c|}{} & $(0,1)$ & $24/65$ & $6/65$ & $7/13$ & $0$\\\cline{2-6}
         \multicolumn{1}{c|}{} & $(1,0)$ & $23/65$ & $0$ & $14/65$ & $28/65$\\\cline{2-6}
         \multicolumn{1}{c|}{} & $(1,1)$ & $23/260$ & $69/260$ & $42/65$ & $0$\\\hline
    \end{tabular}}
\end{equation}
These types of models are called  (2,2,2) Bell-type scenarios (2 parties, 2 possible inputs each, 2 possible outcomes each) and are the simplest non-trivial empirical models with definite causality (up to relabelling).   (2,2,2) Bell-type scenarios  are the models explored further below, in the main body of this work. Note that this is indeed a compatible family on the given causal scenario as the restriction of the lowerset $\{A\}\in \Lambda(\Omega)$ is well defined, i.e. we have:
\begin{equation}
    \left.e_{(\Omega, (I_A, \{0\}))}\right|_{(\{A\}, I_A)} = \left.e_{(\Omega, (I_A, \{1\}))}\right|_{(\{A\}, I_A)} = \scalebox{.8}{\begin{tabular}{cc|c|c|}
        \multicolumn{2}{c|}{\multirow{2}{*}{$A$}} &\multicolumn{2}{c|}{Output}\\\cline{3-4}
        & & $0$ & $1$ \\\hline
        \multicolumn{1}{c|}{\multirow{2}{*}{\rotatebox{90}{Input~}}} & $0$ & $6/13$ & $7/13$ \\\cline{2-4}
        \multicolumn{1}{c|}{} & $1$ & $23/65$ & $42/65$ \\\hline
   \end{tabular}}
\end{equation}

\subsection{The causal fraction}

So far, we  assumed that the causal order of our system is known, but this is not generally the case. We now describe  how given a known final distribution such as the one depicted in \eqref{eq:finalEMcausal}, one can decide what  the underlying causal order is.  In particular, one define the \emph{causal fraction} of a (final) family of distributions which corresponds to the proportion of the model which is compatible with a given causal order.  In the phenomena we are modelling, i.e. semantic ambiguity in natural language,   the causal order of the system is   in general unknown and so we have to start from the set of probability distributions over all events and only decide later which of these is the most likely causal order.

\paragraph{}If a model $e$ is not fully compatible with a given causal order $\Omega$, we can calculate how much of the model can be explained by the causal order $\Omega$. By definition, this is defined as the maximal $\gamma\in [0,1]$ s.t.:
\begin{equation}
\label{eq:causal}
    \gamma \cdot e^\Omega \preceq e 
\end{equation}
where the partial order $\preceq$ is defined component-wise on empirical models defined on the same causal scenarios. Equivalently, this is to say that the causal fraction corresponds to the maximal proportion of the model which can be explained by an empirical model compatible with a given causal scenario.  

In general, finding $\gamma$ in \ref{eq:causal} is a hard optimisation problem, as one needs to consider all empirical models $e^\Omega$ compatible with the causal order $\Omega$.  We here prove that there are simpler ways to approximate, using \ref{eq:causual2}, or in some cases even calculate the causal fraction (see Proposition \ref{prop:2}).

\begin{prop}
    For a family of probability distributions where the causal order is not known, an upper bound of the causal fraction can  be calculated as follows\footnote{Note: the order of the restrictions is from left to right.}:
\begin{equation}
\label{eq:causual2}
    \gamma \leq \min_{U,V} 1 - \left|\left.\left.e_{\underline{i}}\right|_U\right|_{U\cap V}\left(\underline{o}\right) - \left.\left.e_{\underline{i}}\right|_V\right|_{U\cap V}\left(\underline{o}\right)\right|
\end{equation}
where $\left.\left.e_{\underline{i}}\right|_U\right|_{U\cap V}$ corresponds to the restriction of $e_{\underline{i}}$ to first $U$ and then from $U$ to $U\cap V$ (and similarly for $\left.\left.e_{\underline{i}}\right|_V\right|_{U\cap V}$).
\end{prop}
\begin{proof}
    For every causal empirical causal model $e^\Omega$ w.r.t. a causal scenario $\Sigma = \left(\Omega, \underline{I}, \underline{O}\right)$, if we have $\gamma \cdot e^\Omega \preceq e$, then both:
    \begin{equation}
        \gamma \cdot \left.e^\Omega_{\underline{i}}\right|_{U\cap V}(\underline{o})\leq \left.\left.e_{\underline{i}}\right|_U\right|_{U\cap V}(\underline{o})
    \end{equation}
    and
    \begin{equation}
        \gamma \cdot\left.e^\Omega_{\underline{i}}\right|_{U\cap V}(\underline{o})\leq \left.\left.e_{\underline{i}}\right|_V\right|_{U\cap V}(\underline{o})
    \end{equation}
    So:
    \begin{equation}\label{eq:p1e1}
        \gamma \cdot \left.e^\Omega_{\underline{i}}\right|_{U\cap V}(\underline{o})\leq \min_{X\in \{U,V\}}\left.\left.e_{\underline{i}}\right|_X\right|_{U\cap V}(\underline{o})
    \end{equation}
    Now, since $e_{\underline{i}}$ are probability distributions:
    \begin{equation}
        1 - \left.\left.e_{\underline{i}}\right|_X\right|_{U\cap V}(\underline{o}) = \sum_{\underline{o'}\neq \underline{o}} \left.\left.e_{\underline{i}}\right|_X\right|_{U\cap V}(\underline{o'})
    \end{equation}
    and similarly for $e^\Omega$. Therefore, using $\gamma e^\Omega\preceq e$ once again:
    \begin{equation}\label{eq:p1e2}
        \gamma \left(1 - \left.\left.e^\Omega_{\underline{i}}\right|_X\right|_{U\cap V}(\underline{o})\right) \leq \min_{X\in \{U,V\}} 1 - \left.\left.e_{\underline{i}}\right|_X\right|_{U\cap V}(\underline{o}) = 1 - \max_{X\in \{U,V\}}\left.\left.e_{\underline{i}}\right|_X\right|_{U\cap V}(\underline{o})
    \end{equation}
    Then, writing $m_- = \min_{X\in \{U,V\}}\left.\left.e_{\underline{i}}\right|_X\right|_{U\cap V}(\underline{o})$ and $m_+ =\max_{X\in \{U,V\}}\left.\left.e_{\underline{i}}\right|_X\right|_{U\cap V}(\underline{o})$ for simplicity, we use \eqref{eq:p1e1} and \eqref{eq:p1e2} to get:
    \begin{equation}
        \gamma \leq 1 - m_+ + m_-
    \end{equation}
    Now, using binary minima and maxima this reduces to:
    \begin{equation}
        \gamma \leq 1 - \left|\left.\left.e_{\underline{i}}\right|_U\right|_{U\cap V}\left(\underline{o}\right) - \left.\left.e_{\underline{i}}\right|_V\right|_{U\cap V}\left(\underline{o}\right)\right|
    \end{equation}
    And since this has to be the case for all $U, V \in \mathcal{L}$, the claimed inequality has to hold.
\end{proof}

\noindent
In certain cases, such as the models described in Section \ref{subsubsec:BelltypeEx}, the above inequality  becomes an equality as the upper bound is attained.  This is expressed and proven below.\footnote{We believe  that this equality in fact holds for a larger range of systems. Proving a more general version of the proposition is left to future work.}

\begin{prop} \label{prop:2}
For the causal order $A\to B$ in a (2,2,2) Bell-type scenario, the causal fraction is given by:
\begin{equation}\label{eq:causalF}
    \gamma = \min_{i_A\in \{0,1\}, \underline{o}\in \{0,1\}} 1 - \left|\left.\left.e_{(i_A, 0)}\right|_{A\to B}\right|_{A}\left(\underline{o}\right) - \left.\left.e_{(i_A, 1)}\right|_{A\to B}\right|_{A}\left(\underline{o}\right)\right|
\end{equation}
\end{prop}

\begin{proof}
    Let's describe a construction of a causal empirical model $e^\Omega$ which  satisfies $\gamma\cdot e^\Omega\preceq e$, for any given (2,2,2) Bell-type  model $e$, where $\gamma$ is given as in \eqref{eq:causalF}.

    We start by constructing a probability distribution for the event $A$ as follows. For any $i_A\in I_A$, we select $o^*_A \in O_A$ s.t.:
    \begin{equation}
        \min_{i_B\in I_B}\left.\left.e_{\left(i_A,i_B\right)}\right|_{A\to B}\right|_{A}(o^*_A) = \min_{o_A\in O_A}\min_{i_B\in I_B}\left.\left.e_{\left(i_A,i_B\right)}\right|_{A\to B}\right|_{A}(o_A)
    \end{equation}
    and set:
    \begin{equation}
        \left.e^\Omega_{(i_A, i_B)}\right|_{A}\left(o^*_A\right) = \frac{\min_{i_B\in I_B} \left.\left.e_{(i_A, i_B)}\right|_{A\to B}\right|_{A}\left(o^*_A\right)}{\gamma}
    \end{equation}
    and $\left.e^\Omega_{(i_A, i_B)}\right|_{A}\left(\neg o^*_A\right) = 1 - \left.e^\Omega_{(i_A, i_B)}\right|_{A}\left(o^*_A\right)$. Then we have:
    \begin{equation}
        \gamma \cdot \left.\left.e^\Omega_{(i_A, i_B)}\right|_{A\to B}\right|_{A}\left(o_A\right) \leq \left.\left.e_{(i_A, i_B)}\right|_{A\to B}\right|_{A}\left(o_A\right)
    \end{equation}
    for all $\left(i_A, i_B\right)\in I_A\times I_B$, and for all possible outcome $o_A\in O_A$.
    
    One can then extend this distribution to the lowerset $A\to B = \Omega$ by setting, for example:
    \begin{equation}
        e^\Omega_{(i_A, i_B)} \left(o_A, o_B\right) = \frac{\left.e^\Omega_{(i_A, i_B)}\right|_{A\to B}\left(o_A, o_B\right)}{\left.\left.e_{(i_A, i_B)}\right|_{A\to B}\right|_{A}\left(o_A\right)}\left.\left.e^\Omega_{(i_A, i_B)}\right|_{A\to B}\right|_{A}\left(o_A\right)
    \end{equation} 

    It is routine to check that this construction leads to a valid empirical model $e^\Omega$, which does indeed satisfy $\gamma\cdot e^\Omega \preceq e$.
\end{proof}

\subsubsection{Example} Let's consider another (2,2,2) Bell-type  example, where this time the final family of distributions are:
\begin{equation}
    e=\scalebox{.8}{
    \begin{tabular}{cc|c|c|c|c|}
        \multicolumn{2}{c|}{\multirow{2}{*}{$(A,B)$}} &\multicolumn{4}{c|}{Output}\\\cline{3-6}
        & & $(0,0)$ & $(0,1)$ &$(1,0)$ & $(1,1)$ \\\hline
        \multicolumn{1}{c|}{\multirow{4}{*}{\rotatebox{90}{Input}}} & $(0,0)$ & $0$ & $1/7$ & $0$ & $6/7$\\\cline{2-6}
        \multicolumn{1}{c|}{} & $(0,1)$ & $2/3$ & $1/6$ & $1/6$ & $0$\\\cline{2-6}
        \multicolumn{1}{c|}{} & $(1,0)$ & $1/4$ & $0$ & $1/4$ & $1/2$\\\cline{2-6}
        \multicolumn{1}{c|}{} & $(1,1)$ & $1/5$ & $3/5$ & $1/5$ & $0$\\\hline
   \end{tabular}}
\end{equation}

The marginal distributions for the two choices of inputs for $B$ are given by:
\begin{equation}
    \scalebox{.8}{
    \begin{tabular}{cc|c|c|}
        \multicolumn{2}{c|}{\multirow{2}{*}{$A$}} &\multicolumn{2}{c|}{Output}\\\cline{3-4}
        & & $0$ & $1$ \\\hline
        \multicolumn{1}{c|}{\multirow{2}{*}{\rotatebox{90}{Input~}}} & $0$ & $1/7$ & $6/7$ \\\cline{2-4}
        \multicolumn{1}{c|}{} & $1$ & $1/4$ & $3/4$ \\\hline
   \end{tabular}}
\end{equation}
if $i_B = 0$ and:
\begin{equation}
    \scalebox{.8}{
    \begin{tabular}{cc|c|c|}
        \multicolumn{2}{c|}{\multirow{2}{*}{$A$}} &\multicolumn{2}{c|}{Output}\\\cline{3-4}
        & & $0$ & $1$ \\\hline
        \multicolumn{1}{c|}{\multirow{2}{*}{\rotatebox{90}{Input~}}} & $0$ & $5/6$ & $1/6$ \\\cline{2-4}
        \multicolumn{1}{c|}{} & $1$ & $4/5$ & $1/5$ \\\hline
   \end{tabular}}
\end{equation}
if $i_B = 1$. Therefore, the model is not compatible with the causal scenario $A\to B$ and causal fraction can  be calculated as $\gamma=13/42$ using \eqref{eq:causalF}.  In fact, the model in Section \ref{subsubsec:BelltypeEx}, which we will denote as $e^{A\to B}$ does satisfy:
\begin{equation}
    \frac{13}{42} e^{A\to B}\preceq e
\end{equation}

This means that less than 31\% of this scenario can be explained as process where the choice of input of $A$ can influence the output of $B$.

\section{The experiment}

We started from a list of ambiguous nouns (homonymous and polysemous) and list of ambiguous verbs (homonymous and polysemous) and manually selected the verb-noun pairs for which several possible interpretations of both the verb-object and subject-verb phrases were possible. From these, we  randomly selected 50 phrases that had a homonymous verb and a homonymous noun, 50 phrases with a homonymous verb and a polysemous noun, 50 phrases with a polysemous verb and a homonymous noun, and finally 50 phrases with a polysemous verb and a polysemous noun. This resulted in a dataset of 200 ambiguous phrases with an equal number of  different types of ambiguous  (polysemous or homonymous) verbs and nouns. 

We launched this dataset on the Amazon Mechanical Turk (AMT) engine to collect human judgements. AMT workers were  tasked to rate the plausibility of the different interpretations of  the ambiguous phrases of the dataset. Each worker was provided with all  interpretations of each of the words of each phrase and only saw a subset of the dataset with 8 phrases in it. These sub-datasets are referred to as  \emph{HIT} by AMT. The phrases of each HIT only contained either subject-verb or verb-object combinations. We ranged the plausibility scores  over the discrete 0 to 7 interval and had 8 degrees of plausibility. In Psycholinguistics, a 7 grade scale has been deemed as most effective for human subjects, however, in order to avoid randomly chosen and accidental answers due to indecision, we also allowed for an 8th \emph{neutral} grade. We positioned this option in the middle of the scale and designed an 8 grade scale.   Each worker  was thus tasked to choose one of the 8 provided ``scale-description'' degrees of plausibility for each phrase. These were as follows:

\begin{center}
\begin{tabular}{cccc}
0: \text{impossible} & 1: \text{extremely unlikely} & 2: \text{very unlikely} & 3: \text{somewhat unlikely}\\
 4: \text{neutral}& 5: \text{somewhat likely} & 6: \text{very likely} & 7: \text{extremely likely}
\end{tabular}
\end{center}

The annotations  were  used to compute a probability distribution for all the possible interpretations of a phrase. This was done by averaging the scores of all the workers for a particular phrase, and then normalising the obtained average score. After this step, we combined the probability distributions to form ``Bell-type scenarios'', in which we then studied the causality and contextuality of the empirical models. 

Each phrase was  annotated by 25 workers and in  total we had 1250 annotators.  An annotator  could choose to annotate multiple HIT's, and spent on average 10 minutes per HIT. We paid the workers based on the minimum wage in the UK. The probability distributions obtained from the workers' plausibility scales were  used to form 322 (2,2,2) Bell-type scenarios corresponding to subject-verb phrases, and the same number of  verb-object phrases.  

An examples of a  subject-verb phrase was `pitcher threw'. The annotators were provided with all  possible meaning combinations of the phrase, which were as follows:
\begin{itemize}
\item combination 1: `pitcher' is a type of jug and `throw' is the literal action of sending something through the air, e.g. a ball. 
\item combination 2: `pitcher' is a type of jug and `throw' is the figurative action of sending something into a different state, e.g. a shadow. 
\item combination 3: `pitcher' is a baseball player and `throw' is the literal action of sending something through the air, e.g. a ball. 
\item combination 4: `pitcher' is a baseball player and `throw' is the figurative action of sending something into a different state, e.g. a shadow. 
\end{itemize}

A typical annotation for `pitcher threw' was as follows:

\begin{quote} combination 1: 0,   combination 2:  5,   combination 3: 6,  combination 4:  5
\end{quote}

\section{Causal relations in natural language models}
\subsection{Causality of SVO phrases}
We analyse the causality of subject-verb and verb-object phrases separately. Phrases were then combined in (2,2,2) Bell-type scenarios as described in the previous section. In these scenarios, an event  corresponds to choosing a   grammatical  type for a  word, i.,e. subject of a verb or object of a verb or a verb,  as input and then select  an interpretation of that word as output. An example of such an event is  choosing the word \emph{plant} as the subject of a verb (i.e. \emph{plant} will be the input) and then picking the ``factory'' interpretation of it as the output. An example of a subject-verb and object verb empirical model is shown in Fig. \ref{fig:exEM}; in turns, the full dataset can be found in \cite{dataset}.

\begin{figure}[ht]
    \centering
    \begin{subfigure}[c]{.45\linewidth}
        \centering
        \scalebox{0.8}{
        \begin{tabular}{c|c|c|c|c}
             & $(0,0)$ & $(0,1)$ & $(1,0)$ & $(1,1)$\\\hline
            \emph{\small the paper bored} & 0.21 &  0.13 & 0.51 & 0.15\\
            \emph{\small the paper launched} & 0.18 & 0.23 & 0.16 & 0.43\\
            \emph{\small the plant bored} & 0.17 & 0.30 & 0.16 & 0.37\\
            \emph{\small the plant launched} & 0.19 & 0.20 & 0.28 & 0.33
        \end{tabular}}
        \caption{A subject-verb model.}
    \end{subfigure}\quad%
    \begin{subfigure}[c]{.45\linewidth}
        \centering
        \scalebox{0.8}{
        \begin{tabular}{c|c|c|c|c}
             & $(0,0)$ & $(0,1)$ & $(1,0)$ & $(1,1)$\\\hline
            \emph{\small bored the paper} & 0.19 &  0.23 & 0.29 & 0.29\\
            \emph{\small bored the plant} & 0.18 & 0.21 & 0.32 & 0.29\\
            \emph{\small launched the paper} & 0.26 & 0.23 & 0.21 & 0.30\\
            \emph{\small launched the plant} & 0.29 & 0.18 & 0.23 & 0.30
        \end{tabular}}
        \caption{A verb-object model.}
    \end{subfigure}
    \caption{Examples of empirical models. Note: \emph{plant} and \emph{bore} have multiple meanings (respectively the living organism or the factory and making people loose interest or making holes) and \emph{paper} and \emph{launch} have multiple senses (respectively the material or content of an article and literally setting something in motion or starting an activity.)\label{fig:exEM}}
\end{figure}

What we have access to is the final distribution of possible interpretations of a phrase. What we are interested in is whether our empirical models are compatible with a definite causal order and indeed if so, which one. For  each subject-verb model (resp. verb-object model) we have  two events: $S$ and $V$ (resp. $V$ and $O$); these corresponds to choosing the subject and the verb (resp. verb and object) in a given phrase. There are  3 possible definite (acyclic) causal orders associated to each subject-verb phrase, namely $S\to V$, $V\to S$ and $S$ and $V$ being no-signalling (resp. $V\to O$, $O\to V$ and once again having $V$ and $O$ no-signalling in verb-object phrases). 
The obtained causal fractions for all subject-verb and verb-object models are shown in Figs. \ref{subfig:CfAvBSV} and \ref{subfig:CfAvBVO} respectively.

\begin{figure}[ht]
    \centering
    \begin{subfigure}[c]{.45\linewidth}
        \centering
        \includegraphics[width=.8\linewidth]{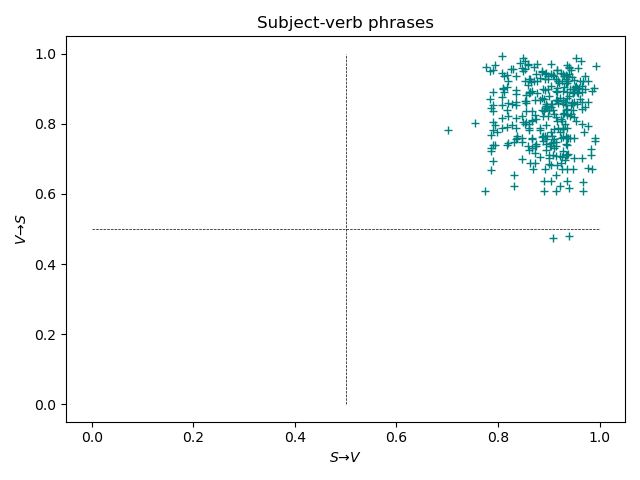}
        \caption{Causal fractions associated with definite causal orders for subject-verb models\label{subfig:CfAvBSV}}
    \end{subfigure}\quad%
    \begin{subfigure}[c]{.45\linewidth}
        \centering
        \includegraphics[width=.8\linewidth]{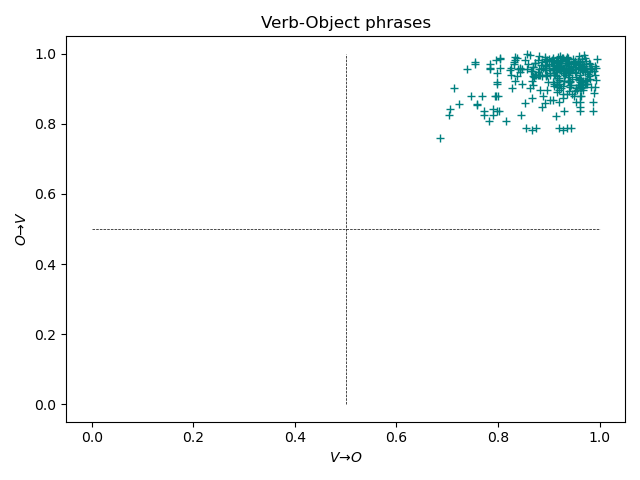}
        \caption{Causal fractions associated with definite causal orders for verb-object models \label{subfig:CfAvBVO}}
    \end{subfigure}
    \caption{}
\end{figure}

What emerged from the data is that subject-verb phrases are predominantly compatible with the $S\to V$ causal order.  All of the models had a causal fraction $>0.7$.  Both the $V\to S$ and the no-signalling fractions achieved  lower causal fraction values, see Fig.\ref{subfig:CfAvBvNSFinSV}. A similar result held for all of the verb-object models: these also  achieved a causal fraction  with the $O\to V$ order  higher than $70\%$, and their other causal fractions reached significantly lower scores, see Fig.\ref{subfig:CfAvBvNSFinVO}. What these results suggest is that the choice of nouns (subject or object) have more influence on which interpretation of the verb is selected, rather than the other way around, i.e. that the choice of verb has more influence on the interpretation of its subject or object. 

\begin{figure}[ht]
    \centering
    \begin{subfigure}[c]{.45\linewidth}
        \centering
        \includegraphics[width=.8\linewidth]{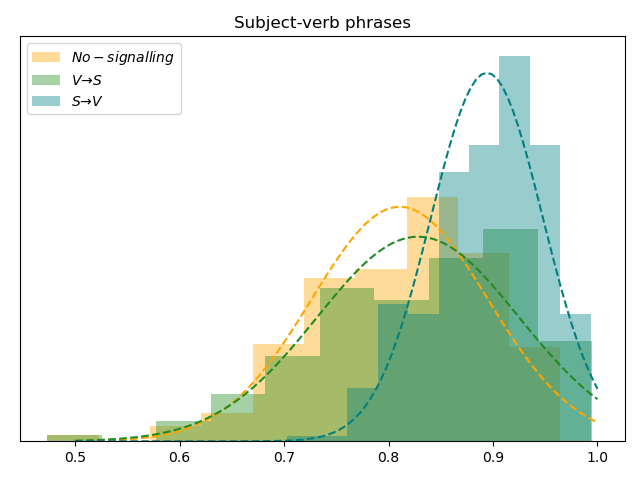}
        \caption{Causal fractions for subject-verb models\label{subfig:CfAvBvNSFinSV}}
    \end{subfigure}\quad%
    \begin{subfigure}[c]{.45\linewidth}
        \centering
        \includegraphics[width=.8\linewidth]{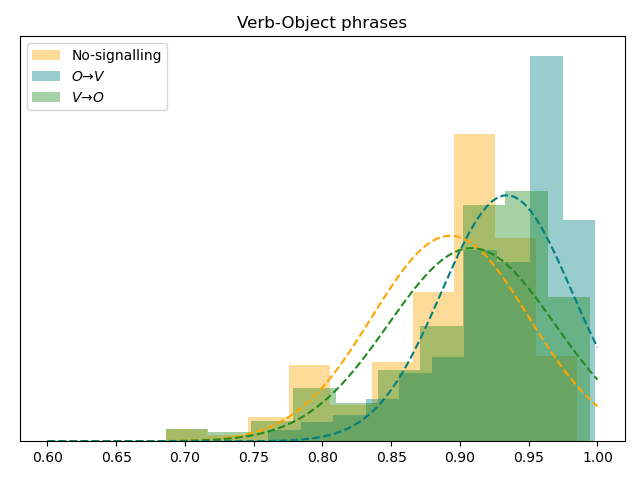}
        \caption{Causal fractions for verb-object models \label{subfig:CfAvBvNSFinVO}}
    \end{subfigure}
    \caption{}
\end{figure}



Finally the causal fractions (in either direction) were also found to be higher for verb-object phrases compared to subject-verb phrases. This would suggests that verb-object phrases are in general easier to disambiguate than subject-verb phrases (see Fig.\ref{fig:SVvVO}).
\begin{figure}[ht]
    \centering
    \includegraphics[width=.35\linewidth]{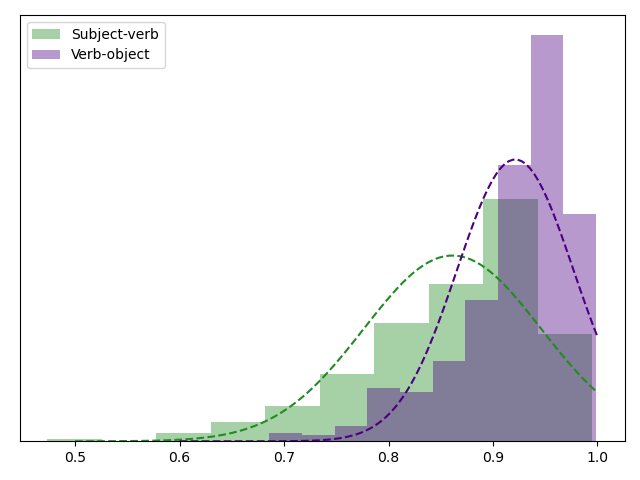}
    \caption{Causal fractions in subject-verb models and in verb-object models.\label{fig:SVvVO}}
\end{figure}

\subsection{Causality and levels of ambiguity}
After determining the causal order of the phrases of our dataset, we would like to  establish a relationship between the type of ambiguity of each word within a phrase, i.e. whether they are homonymous or polysemous,  and the causal fraction of the phrase. In other words, we would like to know whether the type of ambiguity has an effect on the causal fraction. 

After investigating, we found out that the types of ambiguities of the words in the model do not play a major role in the value of the different causal fractions. Indeed, no apparent correlation was observed in the verb-object models, where we computed a Spearman R-coefficient $\rho=0.009$, with a $p$-value$>87\%$.  There  was only a mild effect for the subject-verb models, here the more polysemous the words in a model, the higher the $S\to V$ causal fraction (Spearman R-coefficient $\rho=0.15$, $p$-value$<0.7\%$). These results are depicted   in Fig. \ref{subfig:SVglobalAmbiguity}.

\begin{figure}[ht]
    \centering
    \begin{subfigure}[c]{.45\linewidth}
        \centering
        \includegraphics[width=.8\linewidth]{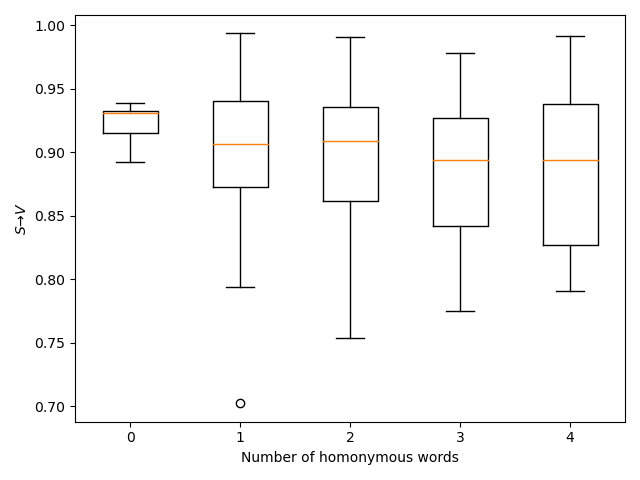}
        \caption{$S\to V$ causal fraction\label{subfig:SVglobalAmbiguity}}
    \end{subfigure}\quad %
    \begin{subfigure}[c]{.45\linewidth}
        \centering
        \includegraphics[width=.8\linewidth]{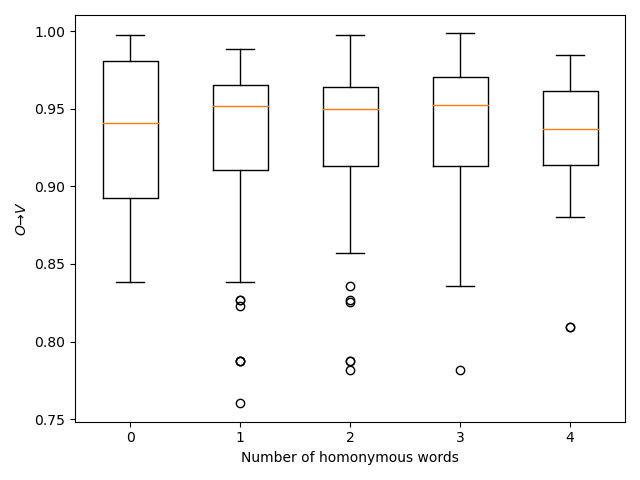}
        \caption{$O\to V$ causal fraction\label{subfig:VOglobalAmbiguity}}
    \end{subfigure}
    \caption{Causal fraction with respect to number of homonymous words in the model}
\end{figure}

The more significant difference was related to the type of ambiguities of the \emph{noun} and the \emph{verb} of the phrase. In both subject-verb and verb-object phrases, the $S\to V$ and $O\to V$ causal fractions were higher when the verb was polysemous (Spearman R-coefficient $\rho=0.17$, $p$-value $<0.2\%$ for subject-verb phrases, R-coefficient $\rho=0.16$, $p$-value $<0.4\%$ for verb-object phrases),  as depicted in see Fig.\ref{subfig:SVlocalAmbiguity} and Fig.\ref{subfig:VOlocalAmbiguityb} respectively. In addition, the causal fraction associated with $O\to V$ was higher whenever objects were homonymous  (Spearman R-coefficient $\rho=0.14$, $p$-value$<2\%$), as depicted in Fig.\ref{subfig:VOlocalAmbiguity},  no such effect was observed for the  subject-verb phrases (Spearman R-coefficient $\rho=0.04$, $p$-value$>50\%$). 

One may not that the Spearman coefficents found above are fairly low ($\rho < 0.2$), which does suggest that the correlations observed are quite mild. However, the $p$-values showed that the correlations claimed in the above paragraph are statistically significant, i.e. it is highly unlikely that no correlation exist between the causal fraction and levels of ambiguity.

\begin{figure}[h]
    \centering
    \begin{subfigure}[c]{.45\linewidth}
        \centering
        \includegraphics[width=.8\linewidth]{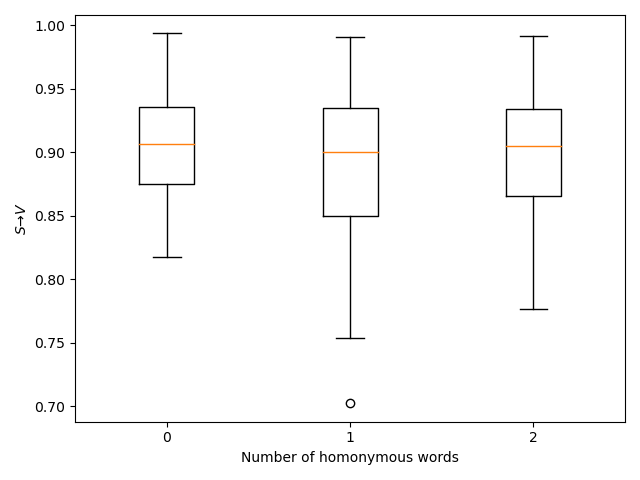}
        \caption{$S\to V$ causal fraction w.r.t. to the number of homonymous nouns.\label{subfig:SVlocalAmbiguityb}}
    \end{subfigure}\quad %
    \begin{subfigure}[c]{.45\linewidth}
        \centering
        \includegraphics[width=.8\linewidth]{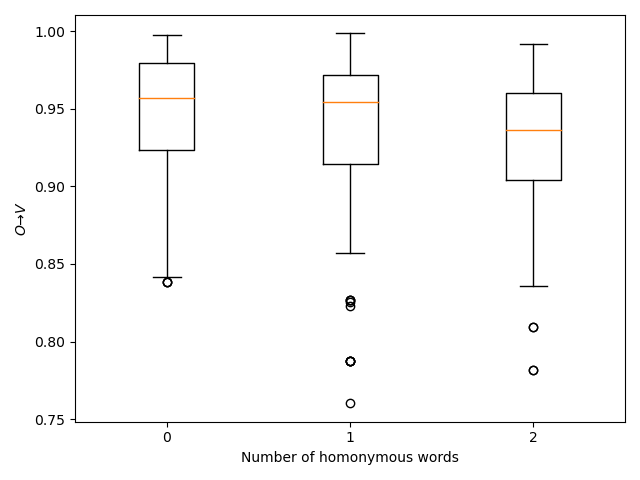}
        \caption{$V\to O$ causal fraction w.r.t. to the number of homonymous verbs.\label{subfig:VOlocalAmbiguityb}}
    \end{subfigure}
    \begin{subfigure}[c]{.45\linewidth}
        \centering
        \includegraphics[width=.8\linewidth]{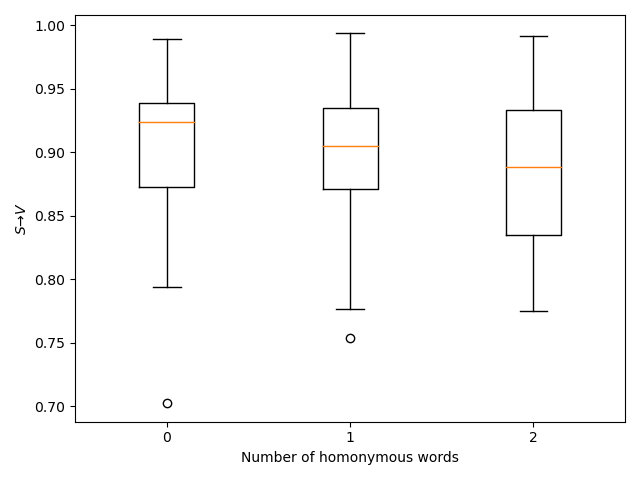}
        \caption{$S\to V$ causal fraction w.r.t. to the number of homonymous verbs.\label{subfig:SVlocalAmbiguity}}
    \end{subfigure}\quad %
    \begin{subfigure}[c]{.45\linewidth}
        \centering
        \includegraphics[width=.8\linewidth]{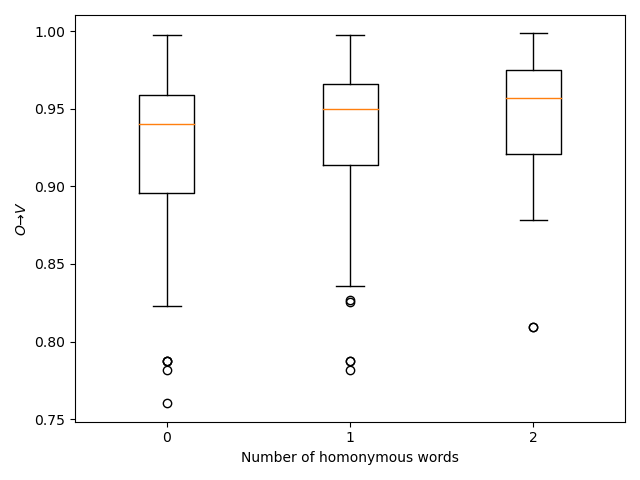}
        \caption{$V\to O$ causal fraction w.r.t. to the number of homonymous nouns.\label{subfig:VOlocalAmbiguity}}
    \end{subfigure}
    \caption{}
\end{figure}

\section{Conclusion}
We believe that our data reflects the findings of the Psycholinguistics and eye-tracking experiments. Indeed, the study of Pickering and Frisson in \cite{PickeringFrisson} observed a delay in disambiguation of ambiguous (transitive) verbs in comparison to the processing of ambiguous nouns. Hence, it would make sense that in a subject-verb or verb-object phrase, where each word is ambiguous, the verb would be disambiguated last, thus explaining the dominance of the $S\to V$ and $O\to V$ causal fractions.

It has also been shown in \cite{PickeringFrisson} that polysemous verbs are disambiguated even later than homonymous verbs. This is consistent with our finding in the subject-verb phrases, where models with several polysemous verbs tend to have a higher causal fraction, as the verb would be disambiguated even later than homonymous verbs would. Similarly, several similar studies have shown that homonymous nouns are (partially) disambiguated instantly \cite{FRAZIER1990181}, whilst polysemous nouns generally require a larger context in order to be (even partially) disambiguated\cite{FRAZIER1990181,FrissonPickering2001b,underspecification}. This fits with our data for the verb-object models, where it was observed that $O\to V$ causal fractions are higher whenever the object had multiple meanings (which is then disambiguated even faster than when the object has multiple senses).

What remains to find out is why the effect of the ambiguity of nouns was different in subject-verb and verb-object models. One  way to interpret this   would be  by taking into account  the difference between  pre and post contexts of the phrases. It was shown in \cite{FRAZIER1990181} that homonymous nouns  were disambiguated a lot faster than polysemous nouns, when the disambiguation context occurs before the target words. This would nicely explain the difference between verb-object and subject-verb phrases. The only possible disambiguation context for nouns in verb-object phrases is the verb, and therefore in such cases, it should be even clearer that the verb would be disambiguated after its object when the latter has multiple meanings. In the case of  disambiguation context for nouns, it was shown that reading times were longer for both homonymous and polysemous nouns. This explains why we did not detect any particular effect in the subject-verb phrases. 


\subsection{Future Work}

The framework developed in \cite{sheafcausality} allows us to distinguish quantum-like (i.e. contextual) and classical processes assuming that a given model is compatible with a definite causal structure. Due to the approximate nature of the probability distribution obtained, i.e. that all the probabilities calculated are not exact, the causal fractions are not exactly $1$, and there are infinitely many sub-distributions that would be compatible with a given causal order. In such cases,  it is not easy to determine whether our empirical model could be contextual or not. Investigating how to extend the model to include these is an interesting future direction.  Also, we would like to collect probabilities using state of the art neural embeddings such as BERT and extending the current study to ambiguities arising from syntax and discourse. Finally, computing casual fragments is much cheaper and quicker than collecting data in labs with eye-tracking equipment and exploring this impact case is our future ambition.  
We would also like to relate our work with related work such as the density matrices model of lexical ambiguity\cite{robin}. Similarly, we could also study other datasets from psychology and behavioural science which showed signalling but no ``true contextuality''\cite{cbdBehaviour}, such as the concept combination dataset of Bruza et. al.\cite{BRUZA201526}.


\bibliographystyle{eptcs}
\bibliography{refs}

\begin{thebibliography}{10}
\providecommand{\bibitemdeclare}[2]{}
\providecommand{\surnamestart}{}
\providecommand{\surnameend}{}
\providecommand{\urlprefix}{Available at }
\providecommand{\url}[1]{\texttt{#1}}
\providecommand{\href}[2]{\texttt{#2}}
\providecommand{\urlalt}[2]{\href{#1}{#2}}
\providecommand{\doi}[1]{doi:\urlalt{http://dx.doi.org/#1}{#1}}
\providecommand{\eprint}[1]{arXiv:\urlalt{https://arxiv.org/abs/#1}{#1}}
\providecommand{\bibinfo}[2]{#2}

\bibitemdeclare{article}{AbramskyBrad}
\bibitem{AbramskyBrad}
\bibinfo{author}{Samson \surnamestart Abramsky\surnameend} \&
  \bibinfo{author}{Adam \surnamestart Brandenburger\surnameend}
  (\bibinfo{year}{2011}): \emph{\bibinfo{title}{The sheaf-theoretic structure
  of non-locality and contextuality}}.
\newblock {\sl \bibinfo{journal}{New J. Phys.}} \bibinfo{volume}{13}, p.
  \bibinfo{pages}{113036},
  \doi{10.1088/1367-2630/13/11/113036}.
\newblock \eprint{https://arxiv.org/abs/1102.0264v7}.

\bibitemdeclare{article}{Bruza2009}
\bibitem{Bruza2009}
\bibinfo{author}{Peter \surnamestart Bruza\surnameend}, \bibinfo{author}{Kirsty
  \surnamestart Kitto\surnameend}, \bibinfo{author}{Douglas \surnamestart
  Nelson\surnameend} \& \bibinfo{author}{Cathy \surnamestart McEvoy\surnameend}
  (\bibinfo{year}{2009}): \emph{\bibinfo{title}{Is there something quantum-like
  about the human mental lexicon?}}
\newblock {\sl \bibinfo{journal}{Journal of mathematical psychology}}
  \bibinfo{volume}{53}(\bibinfo{number}{5}), pp. \bibinfo{pages}{363--377},
  \doi{10.1016/j.jmp.2009.04.004}.
\newblock \urlprefix\url{https://pubmed.ncbi.nlm.nih.gov/20224806}.
\newblock \bibinfo{note}{20224806[pmid]}.

\bibitemdeclare{article}{BRUZA201526}
\bibitem{BRUZA201526}
\bibinfo{author}{Peter~D. \surnamestart Bruza\surnameend},
  \bibinfo{author}{Kirsty \surnamestart Kitto\surnameend},
  \bibinfo{author}{Brentyn~J. \surnamestart Ramm\surnameend} \&
  \bibinfo{author}{Laurianne \surnamestart Sitbon\surnameend}
  (\bibinfo{year}{2015}): \emph{\bibinfo{title}{A probabilistic framework for
  analysing the compositionality of conceptual combinations}}.
\newblock {\sl \bibinfo{journal}{Journal of Mathematical Psychology}}
  \bibinfo{volume}{67}, pp. \bibinfo{pages}{26--38},
  \doi{10.1016/j.jmp.2015.06.002}.
\newblock
  \urlprefix\url{https://www.sciencedirect.com/science/article/pii/S002224961500036X}.

\bibitemdeclare{article}{cbdBehaviour}
\bibitem{cbdBehaviour}
\bibinfo{author}{E.~N. \surnamestart Dzhafarov\surnameend},
  \bibinfo{author}{Ru~\surnamestart Zhang\surnameend} \& \bibinfo{author}{Janne
  \surnamestart Kujala\surnameend} (\bibinfo{year}{2016}):
  \emph{\bibinfo{title}{Is there contextuality in behavioural and social
  systems?}}
\newblock {\sl \bibinfo{journal}{Philosophical Transactions of the Royal
  Society A: Mathematical, Physical and Engineering Sciences}}
  \bibinfo{volume}{374}(\bibinfo{number}{2058}), p. \bibinfo{pages}{20150099},
  \doi{10.1098/rsta.2015.0099}.
\newblock
  \eprint{https://royalsocietypublishing.org/doi/pdf/10.1098/rsta.2015.0099}.

\bibitemdeclare{article}{Dzhafarovetal}
\bibitem{Dzhafarovetal}
\bibinfo{author}{Ehtibar~N. \surnamestart Dzhafarov\surnameend} \&
  \bibinfo{author}{Janne~V. \surnamestart Kujala\surnameend}
  (\bibinfo{year}{2016}): \emph{\bibinfo{title}{{Context–content systems of
  random variables: The Contextuality-by-Default theory}}}.
\newblock {\sl \bibinfo{journal}{Journal of Mathematical Psychology}}
  \bibinfo{volume}{74}, pp. \bibinfo{pages}{11 -- 33},
  \doi{10.1016/j.jmp.2016.04.010}.
\newblock \eprint{https://arxiv.org/abs/1511.03516v6}.
\newblock \bibinfo{note}{Foundations of Probability Theory in Psychology and
  Beyond}.

\bibitemdeclare{article}{FRAZIER1990181}
\bibitem{FRAZIER1990181}
\bibinfo{author}{Lyn \surnamestart Frazier\surnameend} \&
  \bibinfo{author}{Keith \surnamestart Rayner\surnameend}
  (\bibinfo{year}{1990}): \emph{\bibinfo{title}{Taking on semantic commitments:
  Processing multiple meanings vs. multiple senses}}.
\newblock {\sl \bibinfo{journal}{Journal of Memory and Language}}
  \bibinfo{volume}{29}(\bibinfo{number}{2}), pp. \bibinfo{pages}{181--200},
  \doi{10.1016/0749-596X(90)90071-72}.
\newblock
  \urlprefix\url{https://www.sciencedirect.com/science/article/pii/0749596X90900717}.

\bibitemdeclare{article}{underspecification}
\bibitem{underspecification}
\bibinfo{author}{Steven \surnamestart Frisson\surnameend} \&
  \bibinfo{author}{Martin \surnamestart Pickering\surnameend}
  (\bibinfo{year}{2009}): \emph{\bibinfo{title}{{Semantic Underspecification in
  Language Processing}}}.
\newblock {\sl \bibinfo{journal}{Language and Linguistics Compass}}
  \bibinfo{volume}{3}, pp. \bibinfo{pages}{111--127},
  \doi{10.1111/j.1749-818X.2008.00104.x}.

\bibitemdeclare{article}{sheafcausality}
\bibitem{sheafcausality}
\bibinfo{author}{Stefano \surnamestart Gogioso\surnameend} \&
  \bibinfo{author}{Nicola \surnamestart Pinzani\surnameend}
  (\bibinfo{year}{2021}): \emph{\bibinfo{title}{{The Sheaf-Theoretic Structure
  of Definite Causality}}}.
\newblock {\sl \bibinfo{journal}{Electronic Proceedings in Theoretical Computer
  Science}} \bibinfo{volume}{343}, p. \bibinfo{pages}{301–324},
  \doi{10.4204/eptcs.343.13}.
\newblock \eprint{http://dx.doi.org/10.4204/EPTCS.343.13}.

\bibitemdeclare{article}{FrissonPickering2001b}
\bibitem{FrissonPickering2001b}
\bibinfo{author}{Martin \surnamestart Pickering\surnameend} \&
  \bibinfo{author}{Steven \surnamestart Frisson\surnameend}
  (\bibinfo{year}{2001}): \emph{\bibinfo{title}{Obtaining a figurative
  interpretation of a word: support for underspecification}}.
\newblock {\sl \bibinfo{journal}{Metaphor and Symbol}} \bibinfo{volume}{16},
  pp. \bibinfo{pages}{149--171},
  \doi{10.1080/10926488.2001.9678893}

\bibitemdeclare{article}{PickeringFrisson}
\bibitem{PickeringFrisson}
\bibinfo{author}{Martin \surnamestart Pickering\surnameend} \&
  \bibinfo{author}{Steven \surnamestart Frisson\surnameend}
  (\bibinfo{year}{2001}): \emph{\bibinfo{title}{{Processing Ambiguous Verbs:
  Evidence from Eye Movements}}}.
\newblock {\sl \bibinfo{journal}{Journal of experimental psychology. Learning,
  memory, and cognition}} \bibinfo{volume}{27}, pp. \bibinfo{pages}{556--73},
  \doi{10.1037/0278-7393.27.2.556}.

\bibitemdeclare{misc}{robin}
\bibitem{robin}
\bibinfo{author}{Robin \surnamestart Piedeleu\surnameend},
  \bibinfo{author}{Dimitri \surnamestart Kartsaklis\surnameend},
  \bibinfo{author}{Bob \surnamestart Coecke\surnameend} \&
  \bibinfo{author}{Mehrnoosh \surnamestart Sadrzadeh\surnameend}
  (\bibinfo{year}{2015}): \emph{\bibinfo{title}{Open System Categorical Quantum
  Semantics in Natural Language Processing}}, \doi{10.48550/ARXIV.1502.00831}.
\newblock \urlprefix\url{https://arxiv.org/abs/1502.00831}.

\bibitemdeclare{article}{Wangetal-Jour}
\bibitem{Wangetal-Jour}
\bibinfo{author}{D.~\surnamestart Wang\surnameend},
  \bibinfo{author}{M.~\surnamestart Sadrzadeh\surnameend},
  \bibinfo{author}{S.~\surnamestart Abramsky\surnameend} \&
  \bibinfo{author}{V.~\surnamestart Cervantes\surnameend}
  (\bibinfo{year}{2021}): \emph{\bibinfo{title}{{Analysing Ambiguous Nouns and
  Verbs with Quantum Contextuality Tools}}}.
\newblock {\sl \bibinfo{journal}{Journal of Cognitive Science}}
  \bibinfo{volume}{22}(\bibinfo{number}{3}), pp. \bibinfo{pages}{391--420},
  \doi{10.17791/jcs.2021.22.3.391}.

\bibitemdeclare{misc}{Wangetal-Abs}
\bibitem{Wangetal-Abs}
\bibinfo{author}{D.~\surnamestart Wang\surnameend},
  \bibinfo{author}{M.~\surnamestart Sadrzadeh\surnameend},
  \bibinfo{author}{S.~\surnamestart Abramsky\surnameend} \&
  \bibinfo{author}{V.~\surnamestart Cervantes\surnameend}
  (\bibinfo{year}{2021}): \emph{\bibinfo{title}{{In Search of True
  Contextuality in Natural Language}}}.
\newblock \bibinfo{howpublished}{Fourth Workshop on Quantum Contextuality and
  Quantum Mechanics and Beyond (QCQMB)}.
\newblock \bibinfo{note}{Outstanding Paper Award}.

\bibitemdeclare{inproceedings}{Wangetal-Conf}
\bibitem{Wangetal-Conf}
\bibinfo{author}{D.~\surnamestart Wang\surnameend},
  \bibinfo{author}{M.~\surnamestart Sadrzadeh\surnameend},
  \bibinfo{author}{S.~\surnamestart Abramsky\surnameend} \&
  \bibinfo{author}{V.~\surnamestart Cervantes\surnameend}
  (\bibinfo{year}{2021}): \emph{\bibinfo{title}{{On the Quantum-like
  Contextuality of Ambiguous Phrases}}}.
\newblock In: {\sl \bibinfo{booktitle}{Proceedings of the 2021 Workshop on
  Semantic Spaces at the Intersection of NLP, Physics, and Cognitive Science}},
  \bibinfo{publisher}{Association for Computational Linguistics}, p.
  \bibinfo{pages}{42–52}, \doi{10.48550/arXiv.2107.14589}.
\newblock \eprint{https://arxiv.org/abs/2107.14589}.

\bibitemdeclare{misc}{dataset}
\bibitem{dataset}
\bibinfo{author}{Daphne \surnamestart Wang\surnameend} (\bibinfo{year}{2022}):
  \emph{\bibinfo{title}{Causality of lexically ambiguous phrases}}.
\newblock
  \urlprefix\url{https://github.com/wangdaphne/Causality-of-lexically-ambiguous-phrases}.

\end{thebibliography}
\end{document}